\documentclass{article}

\PassOptionsToPackage{numbers}{natbib}

\usepackage[preprint]{neurips_2020}

\usepackage{floatrow}
\usepackage{wrapfig}

\usepackage{amsthm}
\newtheorem{theorem}{Theorem}
\usepackage{algorithm}
\usepackage{algorithmicx}
\usepackage{algpseudocode}

\usepackage{multirow}
\usepackage{graphicx}
\usepackage{xcolor}
\usepackage{subfigure}
\usepackage{amsmath}
\usepackage{booktabs}
\usepackage{floatrow}

\usepackage[utf8]{inputenc} 
\usepackage[T1]{fontenc}    
\usepackage{hyperref}       
\usepackage{url}            
\usepackage{booktabs}       
\usepackage{amsfonts}       
\usepackage{nicefrac}       
\usepackage{microtype}      

\newcommand{\dspnembed}{H_{\textnormal{embed}}}
\newcommand{\dspnenc}{H_{\textnormal{agg}}}

\newcommand{\xhdr}[1]{\vspace{0.0mm}\noindent{\textbf{#1.}}\hspace{0.5mm}}

\usepackage{hyperref}
\definecolor{mylinkcolor}{RGB}{0,0,140}
\hypersetup{colorlinks,allcolors=mylinkcolor,citecolor=mylinkcolor}

\renewcommand{\citet}{\cite}

\usepackage{wrapfig}

\usepackage{tikz}
\usepackage{xcolor}
\usepackage{gensymb}

\usetikzlibrary{shapes, backgrounds, fit, calc, positioning, shapes.geometric}

\definecolor{Blues-4-1}{RGB}{239,243,255}
\definecolor{Blues-4-2}{RGB}{189,215,231}
\definecolor{Blues-4-3}{RGB}{107,174,214}
\definecolor{Blues-4-4}{RGB}{33,113,181}
\definecolor{Oranges-4-1}{RGB}{254,237,222}
\definecolor{Oranges-4-2}{RGB}{253,190,133}
\definecolor{Oranges-4-3}{RGB}{253,141,60}
\definecolor{Oranges-4-4}{RGB}{217,71,1}
\definecolor{Purples-4-1}{RGB}{242,240,247}
\definecolor{Purples-4-2}{RGB}{203,201,226}
\definecolor{Purples-4-3}{RGB}{158,154,200}
\definecolor{Purples-4-4}{RGB}{106,81,163}

\pgfmathsetmacro{\cubex}{0.4}
\pgfmathsetmacro{\cubey}{0.4}
\pgfmathsetmacro{\cubez}{0.4}

\tikzset{
    box/.style={rectangle,draw=black,thick, minimum size=5mm},
    sedge/.style={->, >=stealth},
    cube/.pic = {
    \draw[black, fill=Blues-4-2] (0,0,0) -- ++(-\cubex,0,0) -- ++(0,-\cubey,0) -- ++(\cubex,0,0) -- cycle;
    \draw[black, fill=Blues-4-2] (0,0,0) -- ++(0,0,-\cubez) -- ++(0,-\cubey,0) -- ++(0,0,\cubez) -- cycle;
    \draw[black, fill=Blues-4-2] (0,0,0) -- ++(-\cubex,0,0) -- ++(0,0,-\cubez) -- ++(\cubex,0,0) -- cycle;
    },
    overview/.pic = {
        \node (circ) [scale=1.25, cylinder, shape border rotate=90, draw=black, fill=Oranges-4-1, minimum width=4mm, minimum height=2mm] at (0, 0) {};
        \pic (square) at (1, 0.25) {cube};
        \node (t1) [inner sep=0] at (0.5, 1.2) {Underlying entities};
        \node (t2) at (0.5, 0.8) {in environment};
        \begin{scope}[on background layer]
            \node (env) [inner sep=0.7em, fill=black!3, rounded corners=1mm, fit=(circ) (square) (t1) (t2)] {};
        \end{scope}
        
        \node (img text) at (4.5, 1.2) {Image};
        \node (img) [draw=black, fill=white, scale=5.5] at (4.5, 0.3) {};
        \node (circ) [scale=1, cylinder, shape border rotate=90, draw=black, fill=Oranges-4-1, minimum width=4mm, minimum height=2mm] at (4.2, 0.2) {};
        \pic (square) [scale=0.5] at (4.9, 0.1) {cube};
        
        \node (lb) at (7, 0.3) {\Bigg\{};
        \node (rb) at (8.5, 0.3) {\Bigg\}};
        
        \node [box, scale=0.5, fill=Blues-4-4] at (7.5, 0.25 + 0.3) {}; 
        \node [box, scale=0.5, fill=Blues-4-3] at (7.5, 0 + 0.3) {};
        \node [box, scale=0.5, fill=Purples-4-2] at (7.5, -0.25 + 0.3) {};
        \node [box, scale=0.5, fill=Purples-4-2] at (8, 0.25 + 0.3) {};
        \node [box, scale=0.5, fill=Oranges-4-3] at (8, 0 + 0.3) {};
        \node [box, scale=0.5, fill=Purples-4-3] at (8, -0.25 + 0.3) {};
        
        \node (output) at (10.5, 0.3) {\large Output};
        
        \node at (0.5, -0.7) {unobserved};
        \node at (4.5, -0.7) {observed};
        \node at (7.75, -0.7) {unobserved};
        \node at (10.5, -0.7) {observed};
        
        \begin{scope}[on background layer]
            \node (model) [inner sep=0.3em, fill=blue!3, rounded corners=1mm, fit=(img) (img text) (output)] {};
        \end{scope}
        
        \draw [sedge, shorten >=3mm, shorten <=3mm] (env.east |- img) -- (img);
        \draw [sedge, shorten >=1.5mm, shorten <=2mm] (img) -- (lb) node [midway, above] {\textbf{SRN}};
        \draw [sedge, shorten >=2mm, shorten <=1.5mm] (rb) -- (output);
    },
}

\usepackage[tableposition=top]{caption}
\title{Better Set Representations For Relational Reasoning}

\makeatletter
\newcommand{\printfnsymbol}[1]{%
  \textsuperscript{\@fnsymbol{#1}}%
}
\makeatother

\author{%
  Qian Huang \thanks{equal contribution} \\
  Cornell University \\
  \texttt{qh53@cornell.edu} \\
   \And
  Horace He \printfnsymbol{1} \\
  Cornell University \\
  \texttt{hh498@cornell.edu} \\
   \And
  Abhay Singh \\
  Cornell University \\
  \texttt{as2626@cornell.edu} \\
   \AND
   Yan Zhang \\
   University of Southampton \\
   \texttt{yz5n12@ecs.soton.ac.uk} \\
   \And
  Ser-Nam Lim \\
  Facebook AI \\
   \texttt{sernam@gmail.com} \\
   \And
  Austin Benson \\
  Cornell University \\
   \texttt{arb@cs.cornell.edu} \\
}

\begin{document}
\maketitle

\begin{abstract}
Incorporating relational reasoning into neural networks has greatly expanded their capabilities and scope.
One defining trait of relational reasoning is that it operates on a set of entities, as opposed to standard vector representations. 
Existing end-to-end approaches typically extract entities from inputs by directly interpreting the latent feature representations as a set.
We show that these approaches do not respect set permutational invariance and thus have fundamental representational limitations.
To resolve this limitation, we propose a simple and general network module called a Set Refiner Network (SRN).
We first use synthetic image experiments to demonstrate how our approach effectively
decomposes objects without explicit supervision.
Then, we insert our module into existing relational reasoning models and show that respecting set invariance leads to substantial gains in prediction performance and robustness on several relational reasoning tasks.
%
\end{abstract}



\section{Introduction} \label{sec:intro}
Modern deep learning models perform many tasks well, from speech recognition to object detection.
However, despite their success, a criticism of deep learning is its
limitation to low-level tasks as opposed to more sophisticated reasoning.  This
gap has drawn analogies to the difference in so-called
``System 1'' (i.e., low-level perception and intuitive knowledge)
and ``System 2'' (i.e., reasoning, planning, and imagination) thinking from cognitive psychology~\cite{Bengio2019talk,kahneman2011thinking}.
Proposals for moving towards System 2 reasoning in learning systems
involve creating new abilities for composition, combinatorial generalization, and disentanglement
\cite{Battaglia2018RelationalIB,marra2019integrating,Steenkiste2019AreDR}.

One approach for augmenting neural networks with these capabilities is
performing relational reasoning over structured representations, such as sets or
graphs. This approach is effective in computer vision for 
tasks such as visual
question answering, image captioning, and video
understanding~\cite{Hussein2019VideoGraphRM,Santoro2017ASN,Yang_2019_CVPR}.  For
relational reasoning, these systems are commonly split into two stages: 1)
a perceptual stage that extracts structured sets of vector representations,
intended to correspond to entities from the raw data, and 2) a reasoning stage
that uses these representations. 
As the underlying data is unstructured (e.g.,
images or text), designing end-to-end models that 
generate set representations is challenging. Typical differentiable methods
directly map the input to latent features using a feedforward neural network and
 partition the latent features into a set representation for
downstream reasoning~\cite{Baradel2020COPHYCL,Kipf2020ContrastiveLO,Santoro2017ASN,Zambaldi2018RelationalDR}.

However, this class of existing methods has a fundamental flaw that prevents them from extracting certain desired sets of entities---the responsibility problem~\cite{Zhang2019FSPoolLS}. 
At a high level,
if there exists a continuous map that generates inputs from entities, then any function that can map such inputs to a list representation of the entities must be discontinuous (under a few assumptions that we formalize in Section~\ref{sec:methods}).
For relational reasoning tasks, this implies the perceptual stage must contain discontinuous jumps.
Existing methods use a feedforward neural network to approximate such a discontinuous map, so entities from certain inputs cannot be represented faithfully as shown in Fig.~\ref{fig:responsibility}. We demonstrate this problem extensively in
Sections~\ref{sec:circle_reconstruct} and \ref{sec:sort_of_clevr}.

\begin{figure}[t]
    \centering
    \includegraphics[width=\linewidth]{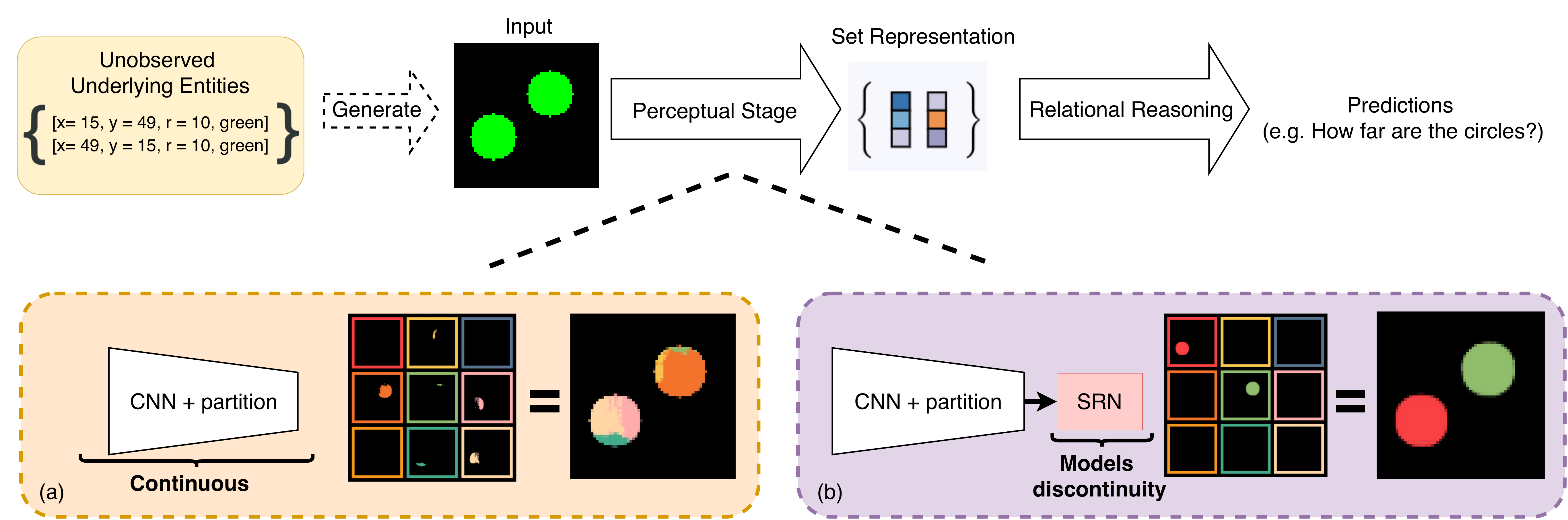}
    \caption{Overview of SRN and responsibility problem. The top row shows an overview of the relational reasoning paradigm. Under the assumptions of Theorem~\ref{theorem}, a perceptual stage that effectively recovers underlying set structure must be discontinuous. The second row shows a visualization of the perceptual stage for (a) existing methods and (b) SRN on a reconstruction task. Each of the 9 color coded panels corresponds to one of the 9 set elements. As existing methods (a) can only represent continuous functions, it is not able to recover the underlying object structure of the image. However, SRN (b) allows us to model discontinuities and thus can recover the underlying object structure.}
    \label{fig:responsibility}
\end{figure}

Here, we introduce the \emph{Set Refiner Network (SRN)}, a novel module that
iteratively refines a proposed set using a set encoder to manage
the aforementioned responsibility problem. 
The main idea of our module is that \emph{instead of directly mapping
an input embedding to a set, we instead find a set representation
that can be encoded to the input embedding}.  
This iterative procedure with simple gradient descent is a better model for the discontinuous jumps required to address the responsibility problem. 
As shown in Fig.~\ref{fig:responsibility}, the SRN can produce set
representations that are ``properly decomposed'', 
meaning that each set element corresponds to the 
entirety of one underlying entity or nothing at all.
In extensive experiments we demonstrate that the SRN
can effectively decompose entities in a controlled setting (Section~\ref{sec:circles}).


Furthermore, we hypothesize that proper decomposition obtained via SRNs can improve both accuracy and robustness of downstream reasoning modules,
analogous to prior work that suggests that disentanglement is desirable for downstream tasks~\cite{Steenkiste2019AreDR}. Intuitively, with a proper decomposition of entities a reasoning module only needs to learn relations between two set
elements to reason about two entities. On the other hand, if an entity is split
amongst multiple set elements, the relational reasoning module needs to learn
more complex relationships to perform the same task.

Indeed, we find that incorporating our SRN into relational
reasoning pipelines in vision, reinforcement learning, and natural language
processing tasks improves predictions substantially. 
On the Sort-of-CLEVR~\cite{Santoro2017ASN} benchmark, simply adding the SRN
to a Relational Network reduces the relative error by over 50\% on the most challenging question category.
We also show that the SRN can improve the robustness of these relational reasoning
modules. For example, even though a Relational Network has high test accuracy on
easy questions within Sort-of-CLEVR, we show that the learned model is
startlingly brittle, producing incorrect answers when the input is perturbed in
a way that should not affect the answer. Again, simply plugging our SRN into the Relational Network resolves this issue, increasing robustness significantly.

\subsection{Related Work}

\xhdr{Iterative Inference}
Iterative inference is the general idea of optimizing a learning procedure to recover latent variables instead of directly mapping from the input to the latent variables in one step~\cite{Marino2018IterativeAI}.
This idea has been used to predict sets in unsupervised object
detection~\cite{Greff2019MultiObjectRL}, general set prediction
\cite{Zhang2019DeepSP}, and energy-based generative
models~\cite{Mordatch2018ConceptLW}. Indeed, previous work has also noted that mapping from a set to a list requires iterative inference to avoid the responsibility problem~\cite{Zhang2019DeepSP,Zhang2019FSPoolLS}. 
However, we are the first to demonstrate that 
1) using iterative inference resolves a fundamental issue in
relational reasoning over raw data, and 
2) the ability for such an approach to
learn objects \emph{without direct supervision or reconstruction as an objective}. This second point prevents previous approaches from easily integrating into existing relational reasoning models. For instance, deep set prediction networks \cite{Zhang2019DeepSP} require an additional loss term that uses the ground truth sets during training, which makes them unsuitable for our setting where ground truth sets are not available.

\xhdr{Pre-trained Modules for Perceptual Reasoning}
An alternative method for perceptual reasoning is to use a
pre-trained object detector~\cite{Yang_2019_CVPR,Yao2018ExploringVR,Yi2018NeuralSymbolicVD}.
An immediate disadvantage is that this requires a comprehensive object detector,
which might require extensive training data in a new application.
However, even with a detector, this approach is limited by a priori imposing what ``objects'' should be, which may be unsuitable for different reasoning tasks. For example, should a crowd of humans be considered as one object, multiple, or the background? This might depend on the task.
For these reasons, we only consider fully differentiable methods in this paper.

\xhdr{Unsupervised Object Detection}
Unsupervised object detection also focuses on object representations from unstructured data~\cite{Burgess2019MONetUS,Greff2019MultiObjectRL,Lin2020SPACEUO}.
These methods effectively decompose scenes without explicit supervision,
but they often require sophisticated architectures 
specific to object detection and segmentation that are difficult to reuse~\cite{Greff2019MultiObjectRL,Veerapaneni2019EntityAI}.
Similar to pre-trained object detectors, these approaches also manually impose a notion of what ``objects'' should be.
In contrast, we propose an isolated module which can easily be inserted into many existing architectures for diverse applications.

\section{Set Refinement Networks for Better Set Representations}
\label{sec:methods}
The general model for relational reasoning over set-structured
representations maps an unstructured data point $X$ (such as an
image) to a set of vectors, which acts as an intermediate set
representation of entities (e.g., objects). 
Then, a module (e.g., for relational reasoning) uses the set of vectors to make a prediction. We write this two-stage pipeline as follows:
\begin{align}
S = G(X) \;\Longrightarrow\; \hat{Y} = F(S). \label{eq:pipeline}
\end{align}
Here, $G$ is a set generator function that maps an input $X$ to a set of vectors $S$, and $F$ is a (usually permutation-invariant) function that maps the set $S$ to a prediction $\hat{Y}$. 
We will typically think of $F$ as a differentiable relational reasoning module.
As a concrete example, in Sort-of-CLEVR experiments in Section~\ref{sec:sort_of_clevr},
$X$ is an image with several entities (colored shapes), $G$ is based on a grid partition of a CNN's feature maps, and the function $F$ tries to answer questions of the form: what is the shape of the entity that is furthest away from the red circle? 

With images and reasoning tasks, it is essential that the
set $S$ is an effective representation for the reasoning module
to make good predictions.
Indeed, a proper disentanglement of the input into meaningful entities has
been found crucial for performance in a variety of downstream 
tasks~\cite{Steenkiste2019AreDR}.
However, as previously mentioned, existing models suffer from a 
``responsibility problem'',  so it is difficult to decompose the input $X$ into a meaningful
set of entities represented by $S$. 


We now formalize this responsibility problem.
Let $\mathcal{S}_{n}^{d}$
be the set of all sets consisting of $n$ $d$-dimensional vectors. Assume a function 
$g\colon \mathcal{S}_{n}^{d} \to \mathcal{X} \subseteq \mathbb{R}^{k}$
that generates inputs from latent sets exists, where 
$g$ is continuous with respect to
the symmetric Chamfer and Euclidean distances.
The following theorem formalizes the inherent discontinuity
in $G$:
\begin{theorem}
\label{theorem}
Let $n, d \geq 2$ and 
$h\colon \mathcal{X} \to \mathbb{R}^{n \times d}$.
If for every $V = \{v_1, \ldots, v_n\} \in \mathcal{S}_{n}^{d}$, 
there exists a permutation $\sigma$ such that
$h(g(\{v_1, \ldots, v_n\})) = 
[v_{\sigma(1)}, v_{\sigma(2)},...,v_{\sigma(n)}]$,
then $h$ is discontinuous.
\end{theorem}
In other words, if our data can truly be generated from underlying entities by a set function $g$, 
and $h$ perfectly captures the entities (in the sense that $h$ maps the input data to a list representation of the underlying entities), then $h$ must be discontinuous.
The proof follows from applying~\cite[Theorem 1]{Zhang2019FSPoolLS}
to $h \circ g$ (see the Appendix for details).
Partitioning an embedding of a feedforward network cannot
effectively model this discontinuity.




In the rest of this section, we develop a general technique, which we call a \emph{Set Refiner Network}, that can be ``tacked on'' to any set generator function $G$ in Eq.~\ref{eq:pipeline} to create better set representations through better modeling of the discontinuity.
We then show in Section~\ref{sec:experiments} how this improves performance in a variety of 
relational reasoning tasks that can be expressed via the pipeline in Eq.~\ref{eq:pipeline}.

\subsection{Methodology}
Existing methods of implementing $G$ by partitioning latent features suffer from the responsibility problem, so the generated set representation is not always properly decomposed. 
Our main idea is to take the output $S_0$ of the set
generator $G$ and ``refine'' $S_0$ to a better decomposed set $S$.
We do so by iteratively improving $S_0$ so that it \emph{encodes}
to the original input.

Given a data point $X$, let $S_0 = G(X)$.
We also generate an data embedding $z = \dspnembed(X)$ as a reference for what information should be contained in $S$. 
In practice, we implement $\dspnembed$ by sharing weights with $G$, e.g., by taking a linear combination of intermediate feature maps (Section~\ref{sec:sort_of_clevr}).

\begin{algorithm}[t]
\begin{algorithmic}[1]



\State $ S_0 = G(X)$  
\Comment{Encode input with the traditional perceptual stage $G$}

\State  $z = \dspnembed(X) $ 
\Comment{Begin SRN}
 \For{$i=1 ... r$}
    \State $L(S_i) = \| \dspnenc(S_{i-1}) - z \|_2^2$ 
    
    \State $S_i = S_{i - 1} - \alpha \frac{\partial L}{\partial S_{i - 1}} $ 
    \Comment{Gradient step}
 \EndFor
 \State $S = S_{r}$ 
 \Comment{}{End of SRN}
 
 \State $\hat Y = F(S)$
 \Comment{Final prediction}
 \end{algorithmic}
 \caption{One forward pass of the Relational Reasoning System with SRN \label{alg:SRN}}
\end{algorithm}

We then address the responsibility problem
with an iterative inference approach, in line with recent research~\cite{Greff2019MultiObjectRL,Veerapaneni2019EntityAI,Zhang2019DeepSP}.
In particular, given the data point embedding $z$, we seek to find a set $S$ \emph{that encodes to $z$}, i.e., we want to solve the following optimization problem
\begin{equation}
\textstyle S = \arg\min_{S'} L(S') = \arg\min_{S'} \| \dspnenc(S') - z \|_2^2 \label{eq:inner_loss},
\end{equation}
where $\dspnenc$ is a permutation-invariant set encoder function parameterized by neural networks.
Finding a set this way forces permutation invariance into the representation
directly, as $\dspnenc$ is a permutation-invariant set function, while also ensuring that the set captures information in the embedding $z$.

In order to train in an end-to-end fashion, we use what we call a \emph{Set Refiner Network (SRN)},
which is defined by the following inner optimization loop:
\begin{equation}
S  = \textnormal{GradientDescent}_{S'}(L(S'), S_0, r),  \label{eq:inner_opt}
\end{equation}
which means that we run gradient descent for the loss in Eq.~\ref{eq:inner_loss}
for $r$ steps, starting from the initial point $S_0$.
Collectively, this produces the set input $S$ for the
function $F$ making predictions for the downstream task (Alg.~\ref{alg:SRN}).
We can also view this entire procedure as a new perceptual stage $G'$. 

An alternative view of the SRN is as a generative model akin to representation inversion~\cite{Engstrom2019LearningPR} or energy-based models~\cite{Mordatch2018ConceptLW}. 
Given a label and an image classifier, for example, one can generate an image that matches a label through iterative inference. 
In our case, given an embedding $z$ along with a set encoder $\dspnenc$ that maps a set to an embedding, we generate a set that matches the embedding through iterative inference. 

The choice of the module $F$ in Eq.~\ref{eq:pipeline} is crucial for the SRN to be able to decompose the latent space into meaningful entities. In particular, $F$ should be a set function that ignores the order of objects,
which is the case in tasks such as relational reasoning.
As in previous work~\cite{Veerapaneni2019EntityAI, Kipf2020ContrastiveLO}, 
we expect that this symmetry assumption will force the network to learn a decomposition. Intuitively, pushing all information into one set element disadvantages the model, while all set elements need to contain similar ``kinds'' of information as they are processed identically. 
Despite this assumption, typical set generation process still fail to enforce permutation invariance, which is fixed by including our SRN.

\xhdr{Graph Refiner Network (GRN) extension}
We can extend the SRN technique to other types of structured
representations. 
Here, we extend it to graph representations.
In this case, the function $G$ in Eq.~\ref{eq:pipeline} produces a graph
instead of a set, and the downstream reasoning module $F$ operates on a graph.
The only change we need to make is to replace the set encoder $\dspnenc$ in SRN
with a graph encoder. Here, we use a Graph Neural Network~\cite{Battaglia2018RelationalIB} that creates a vector
representation of a graph. Furthermore, even if $G$ produces a set, we can
still interpret this to be a graph where the nodes correspond to the set
elements and the edges are initialized to some pre-configured initialization. We
consider the complete graph as an initialization, where the refinement step
seeks to find weights on the edges in the complete graph that are useful for a
downstream reasoning task. We call this approach a \emph{Graph Refiner
  Network (GRN)}.
However, we find that the GRN is difficult to train and usually does not lead to much improvement over SRN, and we only use the GRN in Section~\ref{sec:language}.

\subsection{Explanatory experiments}\label{sec:circles}
\label{sec:circle_reconstruct}

Before turning to more sophisticated learning pipelines in
Section~\ref{sec:experiments}, we first consider the simpler task of image
reconstruction to demonstrate the capabilities of our Set Refiner Networks.
Image reconstruction is a useful sandbox for understanding set
representations which will aid intuition for more complex reasoning tasks. The set representations should contain as much information about the original image as possible and the latent structure can be easily
visualized as a set of objects.

To this end, we construct a synthetic ``Circles Dataset'' for easy control over
the latent structure.  Each image is $64 \times 64$ pixels with RGB channels in
the range 0 to 1 (Fig.~\ref{fig:decomp_a} is an example data point).
An image contains 0 to 10 circles with varying color and size.  
Each circle is fully contained in the image with no overlap between circles of the same color.

\begin{figure}[t]
\RawFloats
\begin{minipage}{0.45\linewidth}
\centering
\subfigure[Input\label{fig:decomp_a}]{\includegraphics[width=0.32\linewidth]{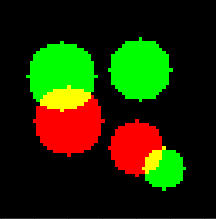}}
\subfigure[SRN\label{fig:decomp_c}]{\includegraphics[width=0.32\linewidth]{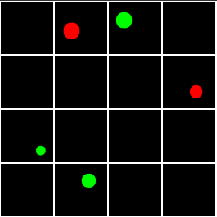}}
\subfigure[Baseline\label{fig:decomp_e}]{\includegraphics[width=0.32\linewidth]{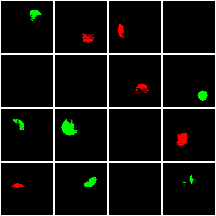}}
\captionof{figure}{\label{fig:decomp}Decomposition results of SRN (b) and
  Baseline (c) models over a sample input (a) in the image reconstruction task
  with the synthetic Circles Dataset. The SRN can successfully decompose
the image.}
\end{minipage}\hfill
\begin{minipage}{0.45\linewidth}
  \captionof{table}{Decomposition success rate on the circles dataset.
    Including the SRN drastically improves the decomposition of the
    images into its constituent entities.}
    \label{decomp_table}
\begin{tabular}{ccc}
\toprule
  & \multicolumn{2}{c}{\% of Images Decomposed} \\ 
 \cmidrule(r){2-3} 
\# of Circles & Baseline & SRN\\
\midrule
1 & 45.1\% & \textbf{95.8}\% \\
2 & 16.5\% & \textbf{88.7}\% \\
3 & 7.8\% & \textbf{79.9}\% \\
4 & 3.2\% & \textbf{72.7}\% \\
5 & 1.9\% & \textbf{65.2}\% \\
\bottomrule
\end{tabular}
\end{minipage}
\end{figure}


We compare two models that implement the basic pipeline in Eq.~\ref{eq:pipeline}.
The models use the same set function $F$ but different set generators $G$.
First, the \emph{Baseline model} implements $G$ using a CNN with groups of channels of the final feature map interpreted as set elements;
this follows the approach of 
differentiable object extractors~\cite{Baradel2020COPHYCL,Kipf2020ContrastiveLO}. 
Second, the \emph{SRN model} extends baseline $G$ by adding the SRN as in Alg.~\ref{alg:SRN}. Specifically, $\dspnembed$ is implemented by flattening the final feature map and passing it through a linear layer. $\dspnenc$ is a set encoder that processes each element individually with a 3-layer
MLP, followed by FSPool~\cite{Zhang2019FSPoolLS}. 
The set function $F$ decodes each element to an image $I_i$ independently through shared transposed-convolution layers,
finally summing the generated images, weighted by their softmax score to ensure their sum lies in $[0, 1]$:
\begin{align}
\textstyle
    I_i = \text{TransposeConvs}(S_i),\;  \quad \hat{Y} = F(S) := \sum_{i = 1}^{\lvert S \rvert}  \text{softmax}_i(I) \cdot \text{sigmoid}(I_i).
\end{align}
We train the model with squared error image reconstruction loss using
the Adam optimizer with learning rate 3e-4. See Appendix B for the full architecture details.

Figure~\ref{fig:decomp} shows an example reconstruction and the images
decoded from the set elements. Although we only provide supervision on
the entire image's reconstruction, SRN naturally decomposes most images into a set of the individual circles. In contrast, the baseline approach does not
decompose the circles properly (Fig.~\ref{fig:decomp_e}).
We also quantitatively measure the quality of the set representations in terms of how well they decompose the circles into separate set elements. We say that a decomposition of an image is \emph{successful} 
if the number of set elements containing objects is equal to the number of circles in the image, where a set
element is considered to contain an object if the decoded image has an $L_{\inf}$ norm greater than 0.1. Using this metric, we find that the SRN is far
more successful at decomposition than the baseline (Table \ref{decomp_table}),
especially when there are multiple circles. We repeat the same experiments with the CLEVR dataset \cite{Johnson2016CLEVRAD} and observe similar results (see Appendix C).

Interpolating circles from one position to another illustrates why we see this difference in decomposition. Both
models enforce symmetric processing of set elements, which causes both models to
exhibit some form of decomposition. However, as the baseline is unable to
overcome the responsibility problem, it cannot use the natural decomposition of the image into objects and instead must decompose the image by location. As we move the circles from left to right, the baseline
gradually hands responsibility for the circle from one set element to the
other. See Appendix B for a visualization
of this. This causes failures of proper decomposition \emph{in between} the two locations. 
On the other hand, the SRN can discontinuously shift
responsibility from one set element to another, thus
properly decomposing the image throughout the interpolation. 

\begin{figure}[t]
\centering
\captionsetup[subfigure]{aboveskip=-2pt,belowskip=-5pt}
\subfigure[Baseline X coordinate]{\includegraphics[width=0.24\linewidth]{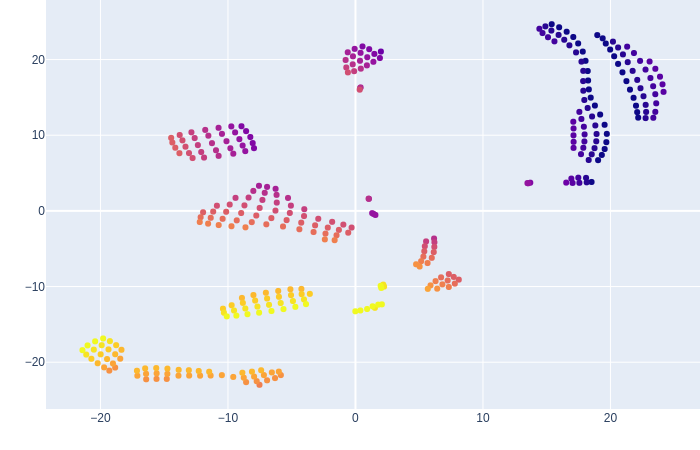}}
\subfigure[Baseline Y coordinate]{\includegraphics[width=0.24\linewidth]{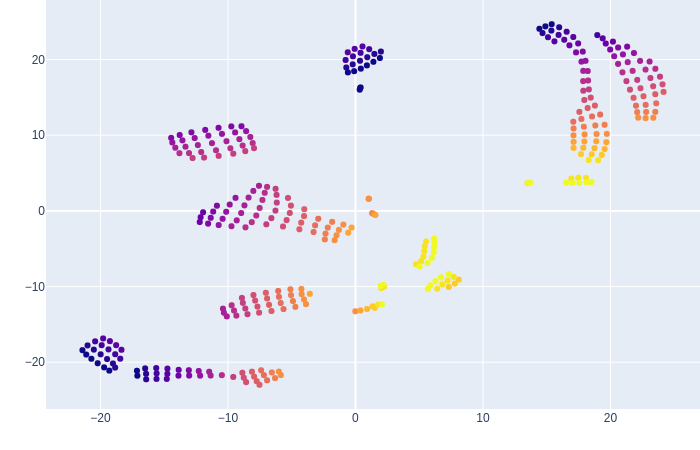}}
\subfigure[SRN X coordinate]{\includegraphics[width=0.24\linewidth]{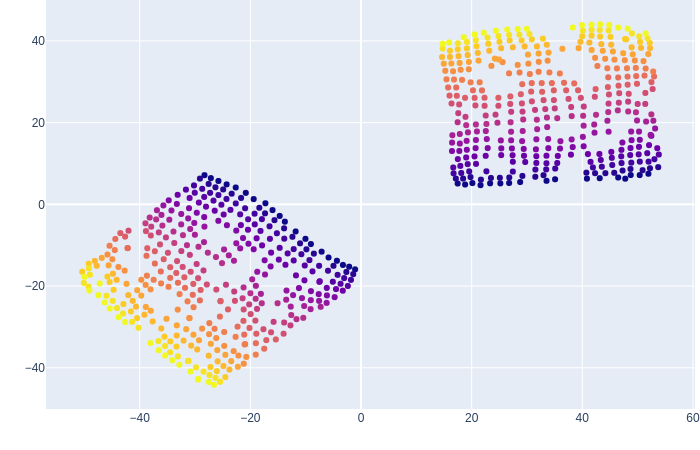}}
\subfigure[SRN Y coordinate]{\includegraphics[width=0.24\linewidth]{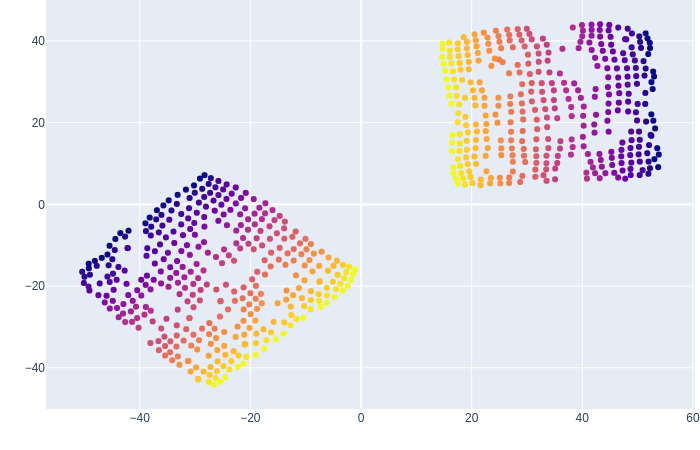}}
\vspace{-\baselineskip}
\caption{t-SNE plots of latent set elements. 
Points are colored by X and Y coordinate values in (a), (c) and (b), (d).
The SRN representations smoothly vary with coordinates and has two distinct clusters corresponding to two colors.}
\label{fig:2dtsne}
\end{figure}

Finally, we visualize how the SRN produces meaningful structure in the latent space.
We grid-sample images with one red or green circle at different
locations and plot the latent space of the set elements corresponding to the
circle using a two-dimensional t-SNE embedding with default scikit-learn
settings (Fig.~\ref{fig:2dtsne}). 
The set elements generated by the
baseline are in discontinuous clusters, while the SRN shows the desired grid
structure with two clear factors of variation --- the X and Y coordinates of
the circles, with two clusters for the two colors. We also visualize coordinates, radius, and color with a
three-dimensional t-SNE embedding in Appendix B, with similar results.
Again, this discrepancy is due to the responsibility problem,
and the smooth latent space of SRN implies that the model
has properly decomposed objects from multi-object images.

\section{Relational Reasoning Experiments}\label{sec:experiments}
In this section, we plug our SRN module into three diverse reasoning
tasks: reasoning about images, reasoning within reinforcement learning, and
reasoning about natural language. 
In all cases, we see that including the SRN refinement step within an existing reasoning pipeline leads to a substantial
increase in performance. 
Moreover, we show how the SRN can make reasoning modules more robust.


\subsection{Relational Reasoning with Images from Sort-of-CLEVR}
\label{sec:sort_of_clevr}
So far, we have shown that the SRN can perform unsupervised object detection in a reconstruction setting, while the baseline cannot. 
We now demonstrate that this advantage can improve both performance and robustness on a relational reasoning task. We use the Sort-of-CLEVR task~\cite{Santoro2017ASN}, which is designed to test relational reasoning capabilities of neural networks. The dataset consists of images of 2D colored shapes along with questions and answers about the objects. The questions are split into non-relational questions and relational questions, and the relational questions are further split into questions of the form ``What is the shape of the object that is the furthest from the red object'', ``What is the shape of the object that is the closest to the grey object'', and ``How many objects are the same shape as the green object''? We call these 3 question categories \emph{Furthest}, \emph{Closest}, and \emph{Count}. We use 100000 images for training, 1000 for validation, and 1000 for test. More details are in Appendix D.

\xhdr{Model} 
We use the same RN architecture as Santoro et al.~\citet{Santoro2017ASN} and insert our SRN between the perceptual end (i.e, the convolutional feature extractor) and the Relational Network, similar to the reconstruction experiment in Sec.~\ref{sec:circles}. All other network and training parameters are kept constant. We train all models for 50 epochs and select the epoch with the best validation accuracy.

\xhdr{Results}
For all experiments, we run 5 trials and report the mean and standard deviation (Table~\ref{tab:sortofclevr}). 
On the most challenging question category (Furthest task),
we improve the performance from 91.4\% to 95.9\%, reducing the error by 52\%. We note that this is the category that requires the most global reasoning from the network. We also demonstrate slight improvements on the Closest task. Overall, we increase average performance from 95\% to 97\%. Non-relational question performance is not reported as both models achieve over 99\% performance.

In order to verify the importance of the inner optimization loop, we take an SRN model trained with 5 inner optimization steps and evaluate with 0 steps. The Count task performance stays at 100\%, the Closest task performance drops modestly from 95\% to 87\%, and the Furthest task drops substantially from 96\% to 63\% (random guessing is 50\%). This provides evidence that the inner optimization step is necessary for the furthest task, which is the one with the most global reasoning.

\xhdr{Robustness}
To verify our original hypothesis about the fundamental representation discontinuity with Relational Networks, we define two ``Robustness'' tasks. In the first one (Robustness Synthetic), we construct an ``easy'' input configuration on which both models provide the correct answer with very high confidence, and create our dataset by translating it to 200 different positions. Details about this configuration are in Appendix D. We test on the ``Furthest'' task, which we believe requires the most global relational reasoning. For a given image, we define the model to be robust on that image if the model provides the correct answer on all of 720 evenly spaced rotations. We rotate the shape coordinates and not the image itself, as all shapes during training are axis-aligned.
We find that the baseline RN is robust on only 1.5\% on this synthetic robustness dataset, while
the model with SRN achieves 96.7\% robustness (Table~\ref{tab:sortofclevr}).

To demonstrate generalization in a setting closer to the training data, in our second robustness task (``Dataset'' task), we generate 1000 new ``easy'' images as a test set, where an image is
``easy'' if the furthest shape is at least twice as far as the second furthest shape. Both models achieve >99\% accuracy on this subset. However, the RN is robust on only 61\% of images, while the RN augmented with SRN is robust on 93\% of images.
Finally, in order to verify that robustness is not simply a result of better overall performance, we evaluate the SRN at an early stage of training on the robustness tasks (Epoch 6). At this point, the SRN's relational accuracy is lower than the RN accuracy at the end of training (93.7\% vs.\ 95.1\% respectively). Still, the SRN performs better on both the synthetic and dataset robustness tasks (55\% and 84\% respectively).

\xhdr{Visualization}
Unlike the image reconstruction task, there is no easy way to visualize the object representations. Thus, for both RN and RN+SRN, we freeze the perceptual stage (i.e., set generator) and 
replace the relational module with the image reconstruction module in Section~\ref{sec:circle_reconstruct}. We overfit on a dataset of 10000 images; this can be viewed as visualizing what meaningful information the latent set elements can contain. The number of trainable parameters is identical for both models.

The quality of the SRN reconstruction is far better than the RN reconstruction. After 50 epochs, the baseline RN has 5000 squared error reconstruction loss on the training set while SRN has 773. 
One reason for this is that the RN will often ``hallucinate'' objects that do not exist (Fig. \ref{fig:rel_reconstruct_samples}). 
For example, even though all input images have one blue shape, the RN will occasionally generate two or three blue shapes. This strongly suggests that the responsibility problem shows up---the baseline RN splits objects into multiple set elements. As each set element is decoded independently and identically, the image decoder is unable to decide which element is the actual circle, producing the hallucinations.

\begin{table}
\centering
\caption{Sort-of-CLEVR performance (mean \% and standard deviation over 5 runs).}\label{tab:sortofclevr}
\vspace{-\baselineskip}
\begin{tabular}{lccc}
\toprule
  & \multicolumn{3}{c}{Performance } \\ 
 \cmidrule(r){2-4}  
Task & RN & RN + SRN & \% Reduction \\
\midrule
Relational & 95.1 {\tiny $\pm 0.4$} & \textbf{96.9 {\tiny $\pm 0.3$}} & 37\\
\midrule
Furthest & 91.4 {\tiny $\pm 0.5$} & \textbf{95.9 {\tiny $\pm 0.4$}} & 52\\
Closest & 93.9 {\tiny $\pm 0.9$} & \textbf{94.8 {\tiny $\pm 0.5$}} & 15\\
Count & \textbf{100 {\tiny $\pm 0.0$}} & \textbf{100 {\tiny $\pm 0.0$}} & N/A\\
\midrule
Robustness (Synthetic) & 1.5 {\tiny $\pm 1.2$} & \textbf{96.7 {\tiny $\pm 2.3$}} & 97\\
Robustness (Dataset) & 61.3 {\tiny $\pm 6.2$} & \textbf{93.2 {\tiny $\pm 3.6$}} & 82\\
\bottomrule
\end{tabular}
\end{table}
\vspace{6mm}
\begin{figure}[t]
\centering
\subfigure[Input image]{\includegraphics[width=0.17\linewidth]{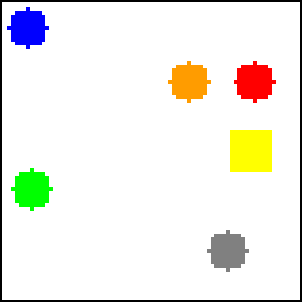}}
\hfill
\subfigure[RN recon.]{\includegraphics[width=0.17\linewidth]{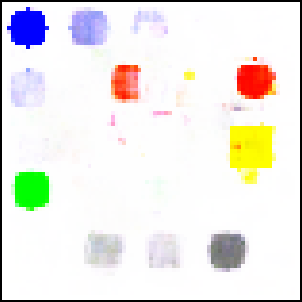}}
\hfill
\subfigure[RN decomp.]{\includegraphics[width=0.17\linewidth]{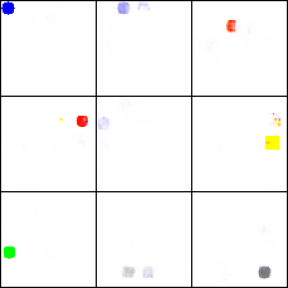}}
\hfill
\subfigure[SRN recon.]{\includegraphics[width=0.17\linewidth]{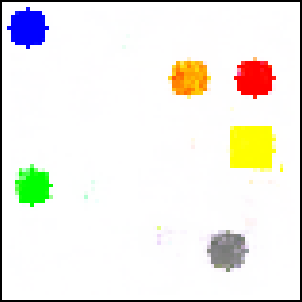}}
\hfill
\subfigure[SRN decomp.]{\includegraphics[width=0.17\linewidth]{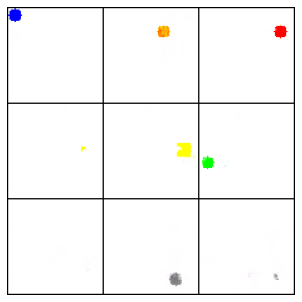}}
\vspace{-\baselineskip}
\caption{Sort-of-CLEVR reconstruction and decompositions (top 9/25 elements ranked by $\ell_1$ norm).\label{fig:rel_reconstruct_samples}
Not only is the baseline RN unable to reconstruct the image properly, it is also unable to decompose the objects, while the SRN properly decomposes the image and reconstructs the image.
}  
\end{figure}

\subsection{World Model Learning}

We demonstrate the advantage of our approach in a more complicated setting, \emph{world model learning}, where the goal is to learn and predict state transitions in a Reinforcement Learning environment.

\xhdr{Model} 
We plug our SRN module into C-SWM \cite{Kipf2020ContrastiveLO}, 
which uses relational reasoning as the core system to learn a world model. In C-SWM, the object representations are extracted by interpreting each channel of the feature maps as an object mask and then passing them through the object encoder individually. The resulting ``state vectors'' are intended to encode the abstract state of each object, which are then used in the Relational Transition Model. This perceptual stage design is similar to the baseline in our image reconstruction experiment in Sec.~\ref{sec:circles}. We plug in our SRN to refine the set of state vectors, with the original state vectors as the initial guess $S_0$ and a linear projection of object masks as the embedding $z$.
The number of set elements is set to be the same as the C-SWM hyperparameter $K$. 
We sweep learning rates for both models and trained for 500 epochs, which led to better baseline performance compared to the original implementation of C-SWM.
More training and architecture details are in Appendix E. 

\xhdr{Results} 
We see performance improvements by incorporating the SRN into C-SWM on 
the Atari Pong (Pong) and Space Invaders (Space) datasets (Table~\ref{tab:world}). These are the only 2 tasks for which C-SWM does not achieve 100\% performance.
As the original paper~\cite{Kipf2020ContrastiveLO} notes, one limitation of C-SWM is that the model cannot disambiguate multiple identical objects and that some kind of iterative inference procedure would be required  to resolve this limitation. Thus, it is perhaps unsurprising that including the SRN provides especially better improvements for the Space dataset, which contains several identical objects. 
In addition, using the SRN has especially notable improvements after 10 action steps. 


\begin{table}[t]
\caption{C-SWM hits at rank 1~\cite{Kipf2020ContrastiveLO} and mean reciprocal rank accuracy with and without SRN on the Atari Pong (Pong) and Space Invaders (Space) datasets (mean and standard deviation over 5 trials). The number in the parenthesis after model names is the number of object slots $K$.}\label{tab:world}
\vspace{-\baselineskip}
\begin{center}
\begin{tabular}{ll ccccccc}
\toprule
&  & \multicolumn{2}{c}{1 Step} & \multicolumn{2}{c}{5 Steps}&\multicolumn{2}{c}{10 Steps} \\ 
\cmidrule(lr){3-4} \cmidrule(lr){5-6} \cmidrule(lr){7-8}
   & Model &  H@1  & MRR &  H@1  & MRR &  H@1  & MRR \\
\midrule
\multirow{4}{*}{\rotatebox{90}{Pong}} 
&C-SWM(5) &$63.6$ {\tiny $\pm 1.6$} & $75.8$ {\tiny $\pm 1.2$} & $27.4$ {\tiny $\pm 1.2$} &  $44.5$ {\tiny $\pm 2.5$} & $17.2$ {\tiny $\pm 1.4$} & $32.2$ {\tiny $\pm 2.1$} \\ 
&SRN (5) &  $\textbf{64.0}$ {\tiny $\pm 1.2$} & $\textbf{76.0}$ {\tiny $\pm 0.8$} & $27.5$ {\tiny $\pm 2.0$} & $43.8$ {\tiny $\pm 2.0$} & $\textbf{18.9}$ {\tiny $\pm 3.0$} & $\textbf{33.4}$ {\tiny $\pm 3.5$} 
\\ 
&C-SWM (3) & $58.4$ {\tiny $\pm 3.1$} & $72.6$ {\tiny $\pm 1.9$} & $24.6$ {\tiny $\pm 2.7$}& $40.8$ {\tiny $\pm 2.6$} & $14.4$ {\tiny $\pm 1.7$}& $29.1$ {\tiny $\pm 1.9$} \\
&SRN (3) & $63.2 $ {\tiny $\pm 1.4$} & $75.4 $ {\tiny $\pm 0.5 $} & $\textbf{28.4 }$ {\tiny $\pm 3.1$} & $\textbf{45.1 }$ {\tiny $\pm 2.3$} & $16.1 $ {\tiny $\pm 2.9$} & $30.9 $ {\tiny $\pm 1.9$} \\
\midrule
\multirow{4}{*}{\rotatebox{90}{Space}}
&C-SWM(5) & $54.8 $ {\tiny $ \pm 2.6$} & $71.7 $ {\tiny $\pm 2.0$} & $46.7 $ {\tiny $\pm 2.9$} & $65.4 $ {\tiny $\pm 2.2$} & $27.1 $ {\tiny $\pm 1.7$} & $45.2 $ {\tiny $\pm 1.9$}\\ 
&SRN  (5) & $64.7 $ {\tiny $\pm 1.1$} & $79.0 $ {\tiny $\pm 0.7$} & $58.0 $ {\tiny $\pm 5.4$} & $74.0 $ {\tiny $\pm 4.2$} & $36.6 $ {\tiny $\pm 4.5$} & $56.6 $ {\tiny $\pm 4.3$}\\
&C-SWM(3) & $\textbf{69.4 }$ {\tiny $\pm 1.1$} & $\textbf{81.9 }$ {\tiny $\pm 0.5$} & $60.0 $ {\tiny $\pm 0.9$} & $75.5 $ {\tiny $\pm 0.6$} & $37.5 $ {\tiny $\pm 4.2$} & $56.7 $ {\tiny $\pm 5.1$} \\
&SRN  (3) & $69.0 $ {\tiny $\pm 1.2$} & $81.6 $ {\tiny $\pm 0.7$}  & $\textbf{63.9 }$ {\tiny $\pm 4.1$} & $\textbf{78.1 }$ {\tiny $\pm 2.8$} & $\textbf{46.5 }$ {\tiny $\pm 10.8$} & $\textbf{65.2 }$ {\tiny $\pm 9.4$}\\

\bottomrule
\end{tabular}
\end{center}
\end{table}

\subsection{Relational Reasoning from Natural Language}\label{sec:language}
Finally, we show that our approach is also applicable over inputs other than images through the text-based relational reasoning task CLUTRR~\cite{Sinha2019CLUTRRAD}. 
In this task, the reasoning task is inferring kinship relationships between characters in short stories that describe the family tree.  
We use the $k = 2,3,4$ datasets~\cite{Sinha2019CLUTRRAD} 
as the training set.
We plug in both SRN and the graph version (GRN)
into the RN baseline with LSTM encoders. 
To adapt to the various number of sentences across the dataset, we fix the set size to 20, with LSTM sentence embeddings padded with zeros as the initial guess $S_0$. The average cell states over all sentences are used as the data embedding $z$. We use the original model, except that we add 0.5 encoder dropout and 1e-4 $\ell_2$-regularization (5e-6 for GRN) to avoid overfitting.
Over 5 runs, for $k = 2,3,4$, the best average performance across epochs for the baseline is 49.7\%, while including the SRN reaches 55.0\%, and using the GRN reaches 56.1\%.

\section{Conclusion}
We have provided substantial evidence that
standard approaches for constructing sets for relational reasoning with neural networks
are fundamentally flawed due to the responsibility problem. 
Our Set Refinement Network module provides a simple solution to this problem with a basic underlying principle: instead of mapping an input to a set, find a set that maps to (an embedding of) the input.
This approach is remarkably effective and can easily be incorporated into a variety of learning pipelines. 
We focused on relational reasoning in this paper, as set structure is fundamental there, but we anticipate that our approach can be used in a variety of domains using set structure.

\section*{Acknowledgments}
This research was supported by NSF Award DMS-1830274, ARO Award W911NF19-1-0057,
and ARO MURI.

\section*{Broader Impact}
As neural networks continue to be deployed in a number of high-stakes tasks,
it is crucial to have a better understanding of their limitations and robustness.
In this paper, we make some progress on both of these issues. 
We showed that end-to-end learning systems that use an intermediate set-structured representation often have difficulty properly decomposing the input into set elements, illustrating that
existing pipelines are not actually creating the representations that they intended to make.
Also, our experiments in Section~\ref{sec:sort_of_clevr} demonstrated that a popular relational reasoning approach in fact learns a brittle model on a fairly simple dataset.
We have presented a first approach at alleviating some of these systemic problems,
and our experiments highlight how to learn more meaningful
set representations, which helps improve robustness in addition to predictive performance.

In addition, our method improves reasoning capabilities of a wide variety of neural networks, which is likely to increase the likelihood of such systems being used in the wild. Most modern ML systems are only relied upon to do ``System 1'' tasks, and we are contributing  to ML models being relied upon for ``System 2'' tasks as well. This has many potential ramifications, both positive and negative, largely related to the general problems of using ML models for real world tasks.

\bibliographystyle{abbrvnat}
\bibliography{example_paper}

\clearpage

\appendix

\section{Theorem of the responsibility problem in relational reasoning}\label{sec:circles_details}



\begin{theorem} Assume that there exists a set function $g: \mathcal{E} \to \mathcal{X} \subseteq \mathbb{R}^{k} $ that generates input from set of entities ($\mathcal{E} = \mathcal{S}_n^d$ is the set of all sets of size n with elements in $\mathbb{R}^d$, $d\geq 2, n \geq 2$), and $g$ is continuous with respect to the symmetric Chamfer and Euclidean distances. Let $h\colon \mathcal{X} \to \mathbb{R}^{n \times d}$. If for every $V = \{v_1, \ldots, v_n\} \in \mathcal{S}_{n}^{d}$, 
there exists a permutation $\sigma$ such that
$h(g(\{v_1, \ldots, v_n\})) = 
[v_{\sigma(1)}, v_{\sigma(2)},...,v_{\sigma(n)}]$,
then $h$ is discontinuous.
\end{theorem}
\begin{proof}
For sake of contradiction, suppose a continuous $h$ exists, then $h \circ g$ is an exact counter example for the responsibility problem theorem in \cite{Zhang2019FSPoolLS}. Thus, such a function cannot exist.
\end{proof}

Generalizations that we can make to help in application in practice:
\begin{enumerate}
    \item  $\mathcal{E}$ can be restricted to subsets of the whole space. One just needs to show the existence of any  smooth curve $\gamma \subseteq \mathcal{E}$ that connects two sets with different permutation. This curve can then be used to replace the circle in the original proof. In the Circles Dataset, we can let $\gamma $ be the rotation of two green circles (while red circles are unchanged), such that $f$ must contain discontinuous jump. Such curve should commonly exist since entities, like input images, generally share the same support (or even assumed to be independent and identically distributed).
    \item This theorem can also be applied by letting $\mathcal{E}$ and $\mathcal{X}$ be a subset of real data. For example, we can apply this theorem to the images with only green circles in the Circles Dataset, such that $f$ must contain discontinuous jump trivially. This reduces the assumptions on the existence of $g$ to only a subset of the dataset.
\end{enumerate}

\section{Circles Dataset Reconstruction Experiment}\label{sec:circles_details}

\subsection{Architecture Details}
\textbf{Set Generator} For the baseline, the set generator $G$ is a standard image encoder derived from Stacked Convolutional Auto-Encoders~\cite{masci2011stacked} and DCGAN ~\cite{Radford2015UnsupervisedRL} with additional processing at the end similar to \citet{Kipf2020ContrastiveLO} and \citet{Baradel2020COPHYCL}:
\begin{enumerate}
    \item  Conv2d layer takes 3 channels as input, 64 filters, 4 kernel size, stride 2, padding 1, no bias, relu activation.
    \item  Conv2d layer takes 64 channels as input, 128 filters, 4 kernel size, stride 2, padding 1, no bias,batch normalized, relu activation.
    \item  Conv2d layer takes 128 channels as input, 256 filters, 4 kernel size, stride 2, padding 1, no bias, batch normalized,relu activation.
    \item  Conv2d layer takes 256 channels as input, 512 filters, 4 kernel size, stride 4, padding 1, no bias,batch normalized, relu activation.
    \item  Reshape to ($\lvert S \rvert$, 2048 / $\lvert S \rvert$), where $\lvert S \rvert$ is the size of the set, and transpose dimension 1,2
    \item  Conv1d layer takes (2048 / $\lvert S \rvert$) channels as input, element dimension number of filters, 1 kernel size.
\end{enumerate}
The output tensor would then be of shape $(b, n, \lvert S \rvert)$, where $b$ is the batch size and $n$ is the dimension of each set element.
This output is interpreted as a set of vector elements for each instance, with the corresponding set size and element dimension pre-specified. This generator is designed to process images with the convolutional layers. The output of step 4 can be seen as a set of feature maps, where each feature map has equal perception field that covers the whole image. We then use step 5 and step 6 to transform this set to the desired shape by grouping feature maps and processing each feature map group individually with shared function. Overall, this makes sure that each set element is produced with the same initial input (the whole image) and with the same architecture, although the weights might be different.

For SRN, the same architecture is used for predicting the innital guess. After that, the embedding is generated by flattening the feature maps and projecting down to a 100-dimensional vector using a fully connected layer. $\dspnenc$ processes each element in the input set individually with a 3-layer MLP with 512 as the hidden dimensions, followed by FSPool \cite{Zhang2019FSPoolLS} with 20 pieces and no relaxation.

\textbf{Set Function} As described in the main paper, the prediction function $F$
decodes each element in the generated set to an image independently through shared transpose-convolution layers. We then aggregate the set of images with self attention. Both baseline and SRN have the same image decoder (TransposeConvs) architecture:
\begin{enumerate}
    \item Fully connected layer, projecting from  element dimension $n$ to feature maps
of shape (1024, 4, 4). Apply batch norm before reshape and use relu activation.
    \item  TransposeConv layer filters 512, kernel size 4, stride
2, padding 1, no bias, batch normalized, relu activation.
    \item  TransposeConv layer filters 256, kernel size 4, stride
2, padding 1, no bias, batch normalized, relu activation.
    \item  TransposeConv layer filters 3, kernel size 4, stride
4, padding 0, no bias.
\end{enumerate}

This architecture first creates objects and then superimposes them for the final reconstruction. This is permutation invariant as the order of the elements in the set does not change the output and allows each element to be interpreted from the individual reconstructed objects, leading to direct disentanglement.

\subsection{Disentanglement Metric Specification}
To measure whether an image was completely disentangled, we examine each of the individual images generated by each of the set elements. If any of the values were more than $.1$ (out of $1$), we considered that set element to be ``non-empty.'' 
If the total number of non-empty set elements matched the total number of circles, we considered the image to be ``completely disentangled.'' 
Although it is possible that this results in false positives (for example, if the model splits one circle into 2 set elements and puts 2 other circles in one set element), we rarely observe this in practice.

\subsection{Interpolation plots}
As described in the main paper, the baseline gradually hands responsibility for the circle from one set element to the other, as we move the circles from left to right. This is visualized in \ref{fig:interpolation}, where the colors represent the responsibilities as in the Figure 1 in the main paper.

\begin{figure}[htb!]
    \centering
    \includegraphics[width=\linewidth]{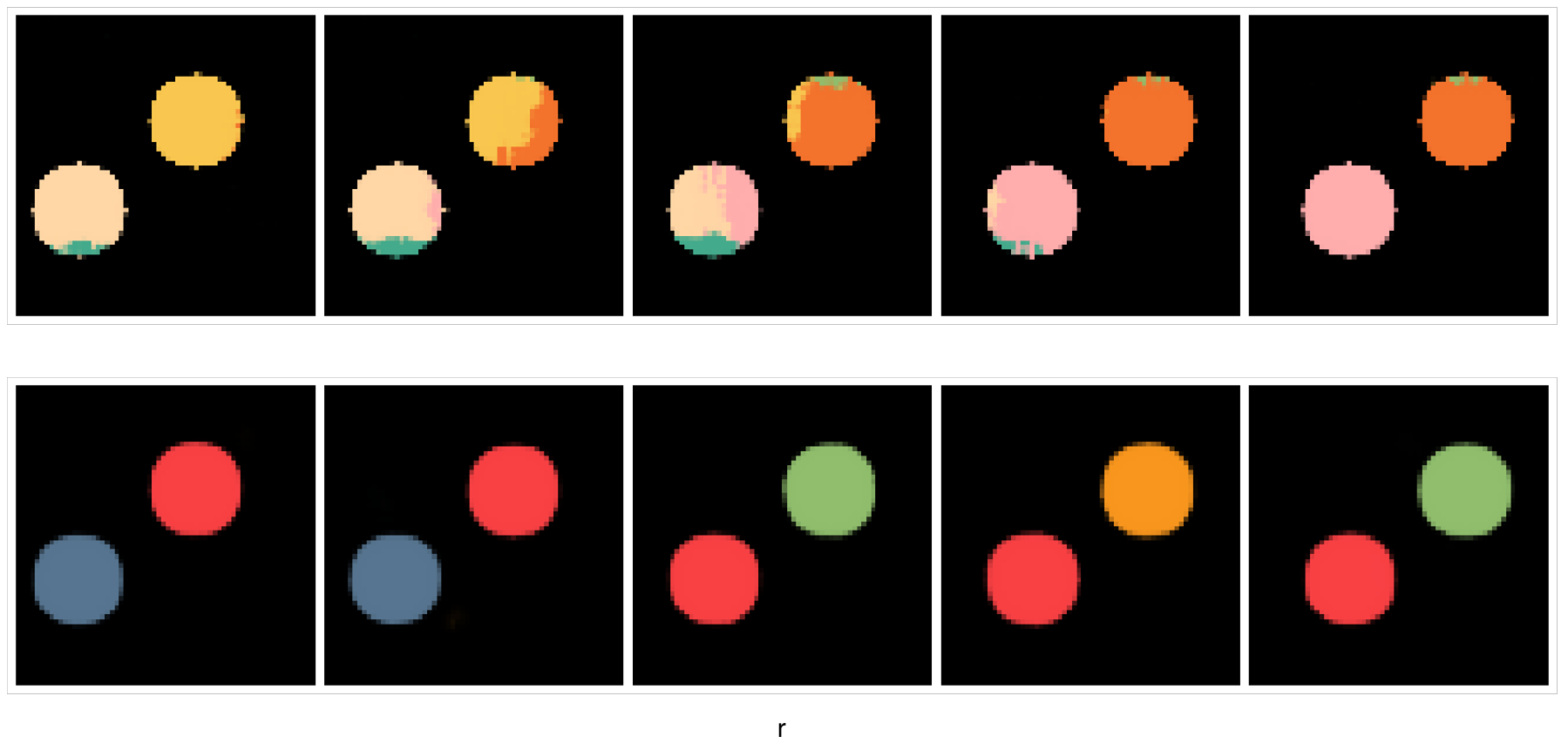}
    \caption{Reconstruction results while input is gradually shifting to the right.}
    \label{fig:interpolation}
\end{figure}

\subsection{3d t-SNE plots}

The three-dimensional t-SNE plots with variations of coordinates, radius and color are shown in Figure \ref{fig:3dtsne}.

\begin{figure*}[htb!]
\centering
\subfigure{\includegraphics[width=0.23\linewidth]{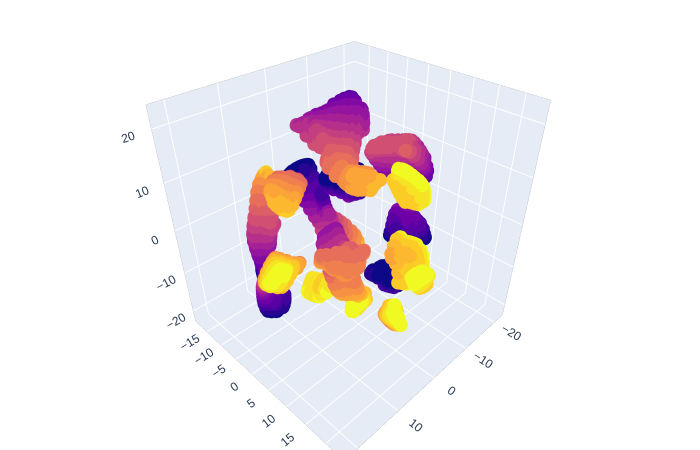}}
\subfigure{\includegraphics[width=0.23\linewidth]{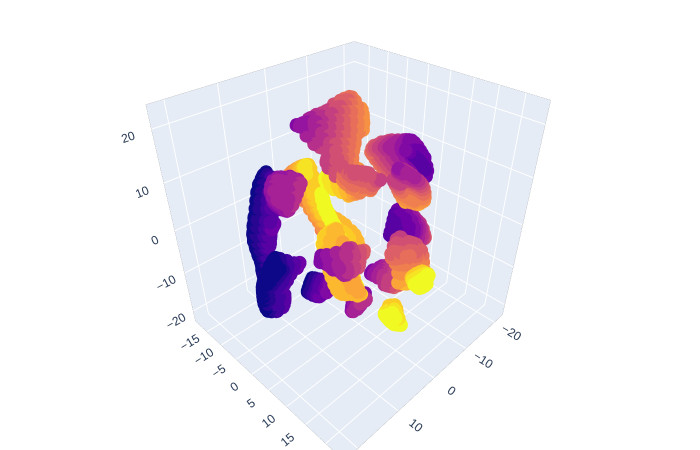}}
\subfigure{\includegraphics[width=0.23\linewidth]{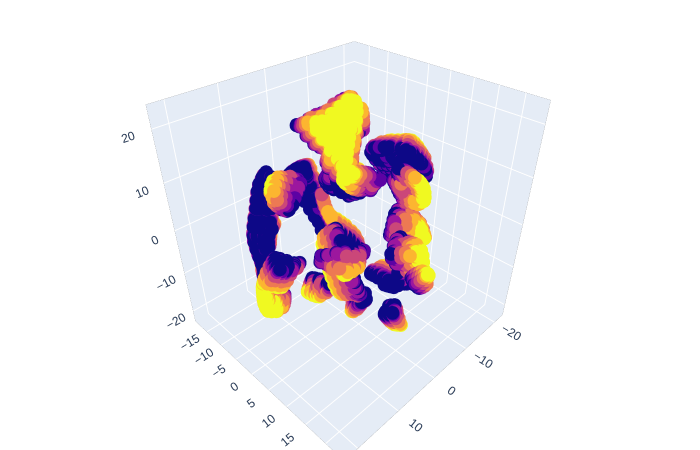}}
\subfigure{\includegraphics[width=0.23\linewidth]{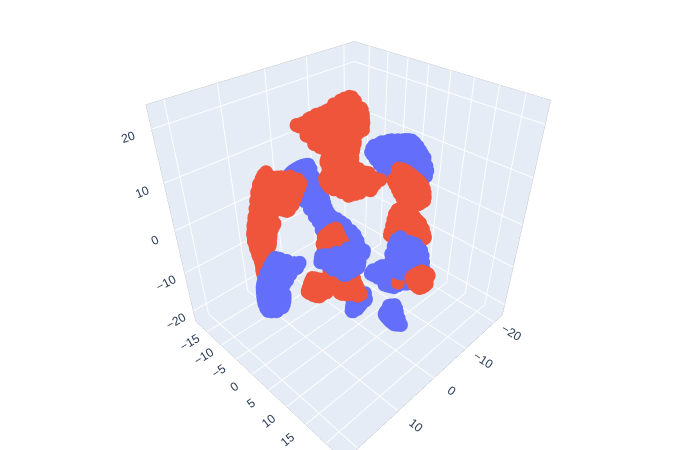}}

\subfigure[X coordinate]{\includegraphics[width=0.23\linewidth]{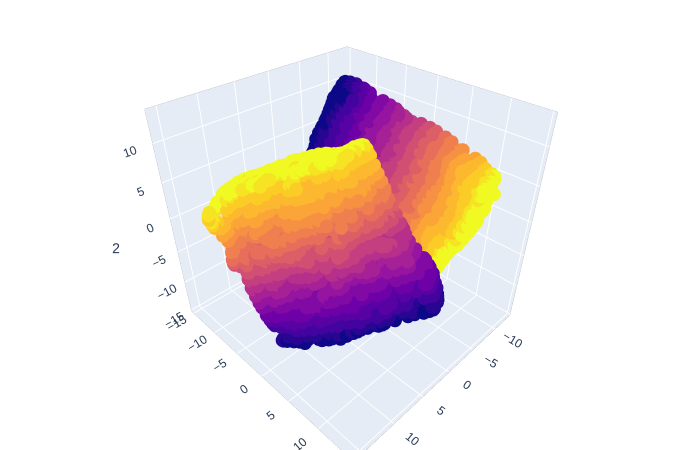}}
\subfigure[Y coordinate]{\includegraphics[width=0.23\linewidth]{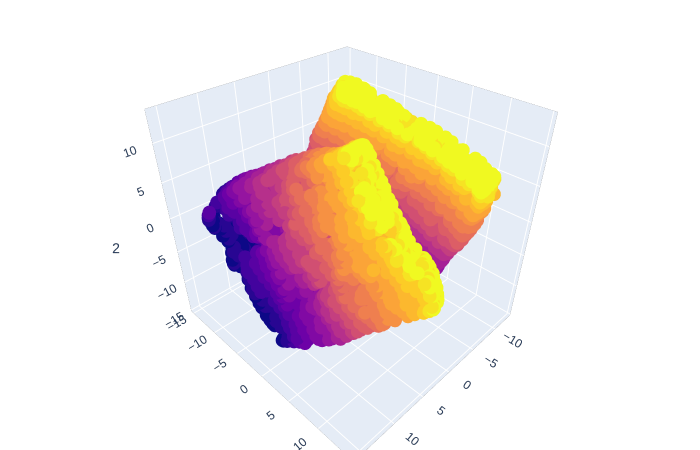}}
\subfigure[Radius]{\includegraphics[width=0.23\linewidth]{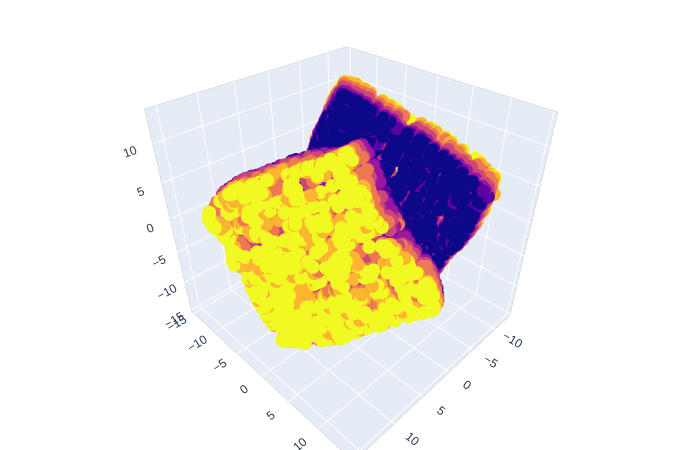}}
\subfigure[Color]{\includegraphics[width=0.23\linewidth]{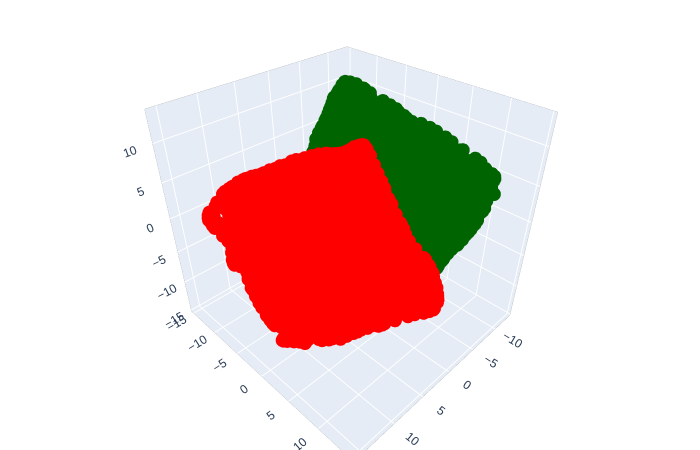}}
\caption{Three-dimensional t-SNE embeddings of set elements generated by baseline (top) and SRN (bottom) method, colored by various parameters of the circles: X coordinate (a), Y coordinate (b), radius (c), and color (d). The representations have clear continuous structure.}
\label{fig:3dtsne}
\end{figure*}

\subsection{Additional Results}

Figures~\ref{fig:circles1} shows 10 randomly sampled
images from the test set along with the latent sets learned by SRN.
As shown in the figure, the decomposition is almost perfect for SRN,
whereas baseline frequently cannot disentangle objects.

\begin{figure*}[htb!]
    \centering
    \includegraphics[width = 0.8\linewidth]{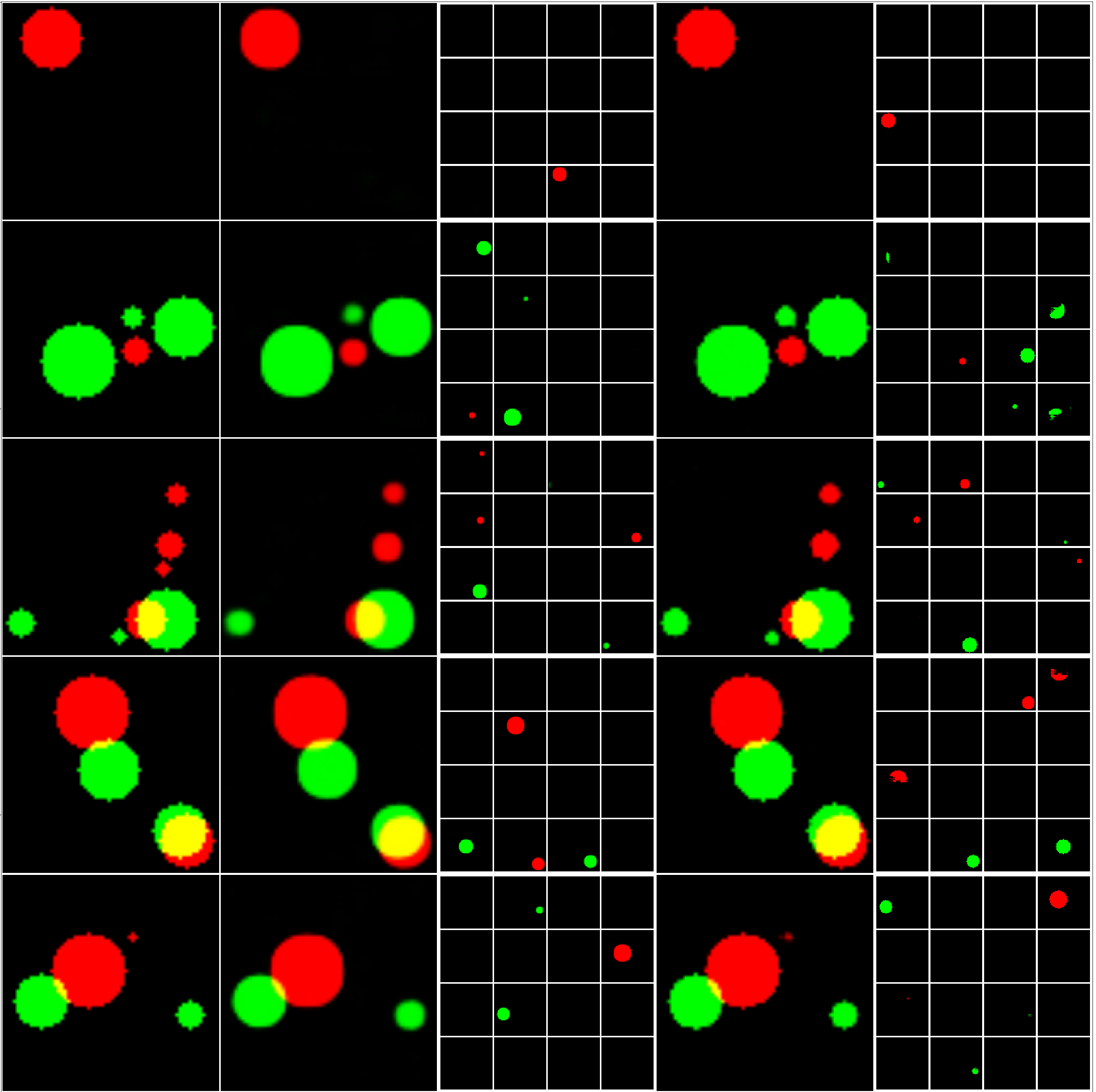}
    \includegraphics[width = 0.8\linewidth]{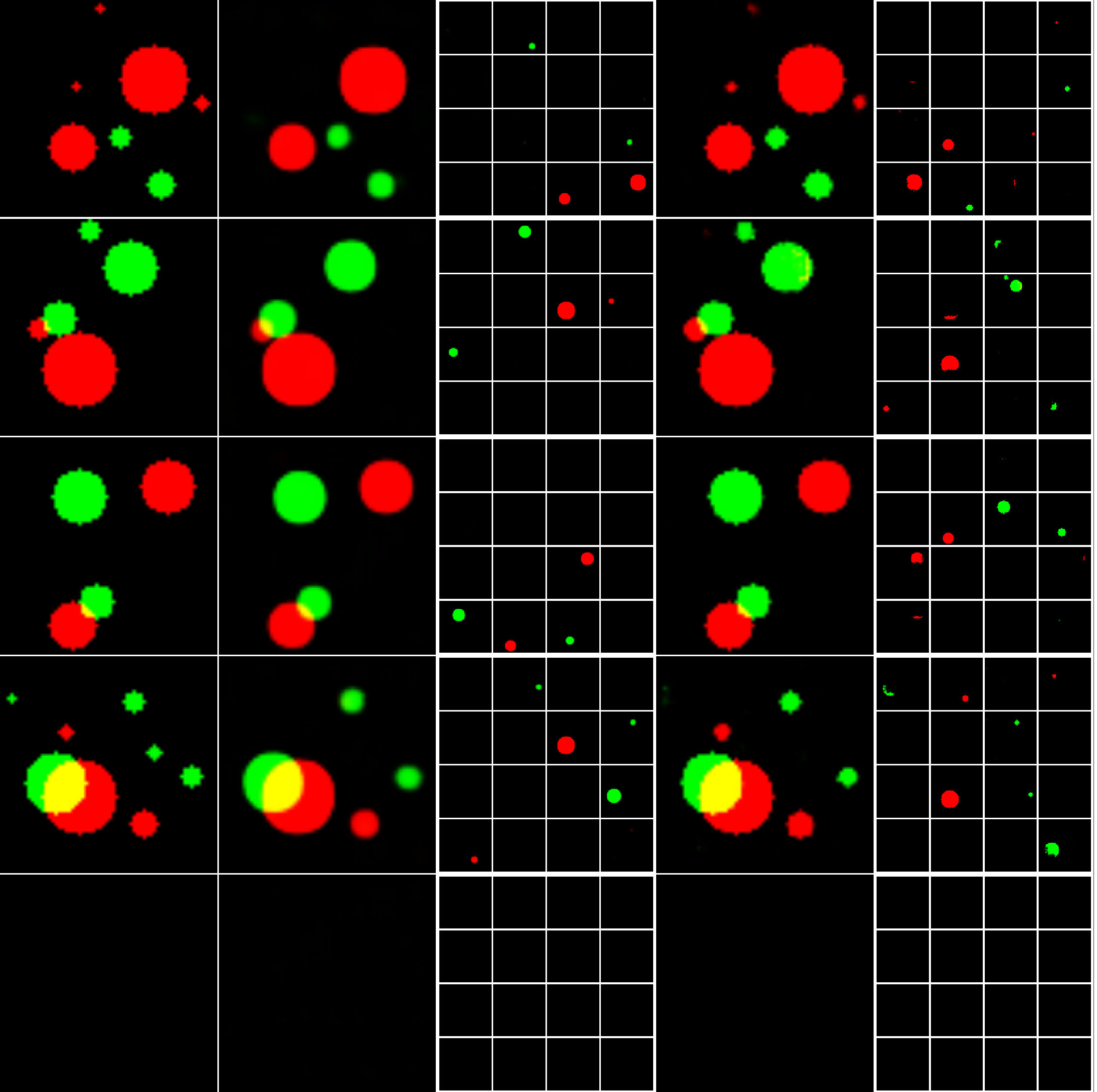}
    \caption{Ten sampled reconstruction and decomposition results. From left to right, column-wise: original images, SRN reconstruction, SRN decomposition, baseline reconstruction, baseline decomposition.}
    \label{fig:circles1}
\end{figure*}


\section{CLEVR Object Reconstruction Results Experiment}\label{app:clevr}

\paragraph{Dataset} We use the CLEVR dataset~\cite{Johnson2016CLEVRAD} to show that such object decomposition holds in more complicated settings. The dataset contains 70,000 training and 15,000 validation images. The original dataset does not contain images of the masked objects, but we generate these through publicly available code.

\paragraph{Model} We again encode images using a CNN that takes 128 $\times$ 128 images as input and has four convolutional layers to produce a 512-dimensional image embedding. The architecture for the CLEVR experiment are nearly the same as ones for the Circles Dataset described above. The only difference is that the projection layers used in image embedding generation and the decoder are modified to adapt to image size of $128 \times 128$.

\begin{table}[t]
\label{iso}
\caption{Intersection over Union (IoU) results for CLEVR image reconstruction.
Our SRN approach out-performs the MLP baseline.}
\label{clevr}
\begin{center}
\begin{tabular}{lcccr}
\toprule
Model Description & MLP & baseline   \\
\midrule
overall IoU & 0.9193 & \textbf{0.9345} \\
per-object IoU & 0.7737 & \textbf{0.8305} \\
\bottomrule
\end{tabular}
\end{center}
\vskip -0.1in
\end{table}

\paragraph{Results} 
Figures~\ref{fig:clevr1}~and~\ref{fig:clevr2} show
10 randomly sampled images from the test set, along with the SRN
and baseline latent sets. 
Figures~\ref{fig:clevr_single},~\ref{fig:clevr_srnset},~and~\ref{fig:clevr_mlpset}
each show one sample in more detail.
Again, SRN disentangled objects and completed occluded part of the objects reasonably, while the baseline failed to disentangle objects.

We also measured the decomposition quality using intersection over union (IoU) score (Table~\ref{clevr}). 
We use two metrics for evaluation: the standard overall intersection-over-union (IoU) and a per-object intersection-over-union. The per-object IoU is the IoU over the Chamfer matching between ground-truth bounding boxes and the bounding boxes of the predicted decomposition. In both cases, we first threshold the pixels valued in $[0,1]$ by $0.01$ (i.e., pixels with all channels smaller than $0.01$ are set to zero).
For the overall IoU case, this threshold is applied on the final reconstruction, and in the per-object case it is applied on each set element. 
Due to rendering noise, it is difficult to obtain instance segmentation masks.
Instead, we algorithmically generate bounding boxes 
with OpenCV's \texttt{findContours()} function applied to the thresholded prediction.
These bounding boxes are then compared to ground truth bounding boxes.
We match the latent set of a prediction (where each element may contain multiple bounding boxes if a prediction fails to disentangle objects) with the set of generated bounding boxes using Chamfer matching, where the cost is the IoU between the pairs of elements: each prediction is matched with the closest ground truth bounding box. 
With this computed assignment, we then compute the average per-object IoU.

The overall IoU is calculated pixel-wise over all objects in the foreground in the reconstructed image. 
The per-object IoU is the IoU over the Chamfer matching between ground-truth bounding boxes and the bounding boxes of the predicted decomposition.
In both cases, SRN performs better than the baseline. The low per-object IoU for the baseline is due to the poor object decomposition and its inability to handle occluded objects.
In contrast, SRN can decompose the scene with the superposition of full objects, including the occluded parts, providing a more meaningful disentangled representation.

\begin{figure*}[htb!]
    \centering
    \includegraphics[width = \linewidth]{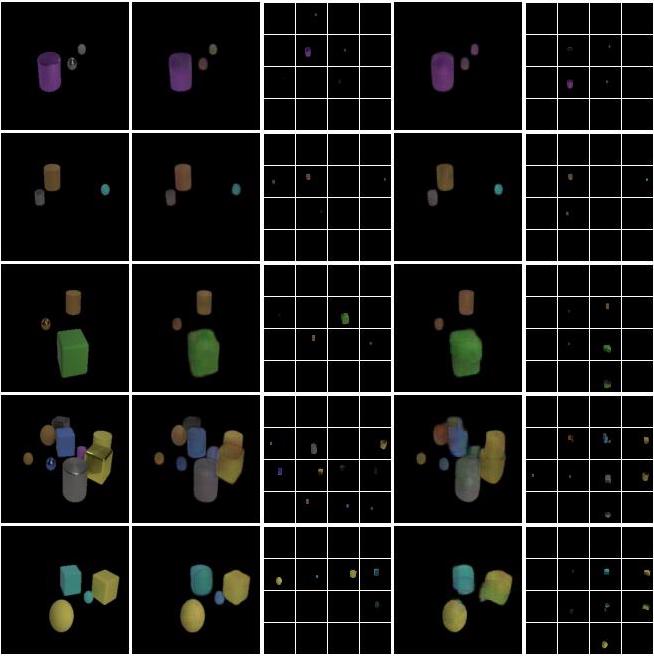}
    \caption{Five sampled reconstruction and decomposition results. From left to right, column-wise: ground truth objects image, SRN reconstruction, SRN decomposition, baseline reconstruction, baseline decomposition}
    \label{fig:clevr1}
\end{figure*}

\begin{figure*}
    \centering
    \includegraphics[width = \linewidth]{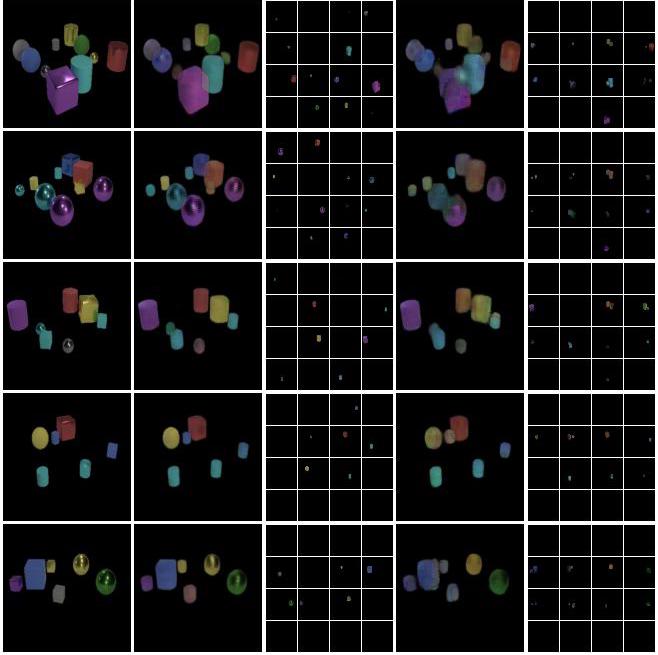}
    \caption{Five sampled reconstruction and decomposition results. From left to right, column-wise: ground truth objects image, SRN reconstruction, SRN decomposition, baseline reconstruction, baseline decomposition}
    \label{fig:clevr2}
\end{figure*}

\begin{figure*}[htb!]
\centering
\subfigure[Ground Truth Image]{\includegraphics[width=0.25\linewidth]{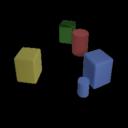}}
\hspace{20pt}
\subfigure[SRN Reconstruction]{\includegraphics[width=0.25\linewidth]{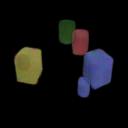}}
\hspace{20pt}
\subfigure[Baseline Reconstruction]{\includegraphics[width=0.25\linewidth]{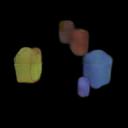}}
\caption{The example ground truth and reconstructions for Figure \ref{fig:clevr_srnset} and \ref{fig:clevr_mlpset}.}
\label{fig:clevr_single}
\end{figure*}

\begin{figure*}[htb!]
    \centering
    \includegraphics[width = 0.9\linewidth]{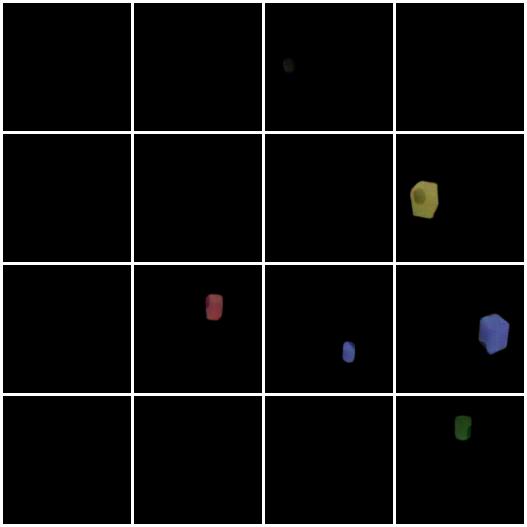}
    \caption{The decomposition of SRN for example in Figure \ref{fig:clevr_single}.}
    \label{fig:clevr_srnset}
\end{figure*}

\begin{figure*}[htb!]
    \centering
    \includegraphics[width = 0.9\linewidth]{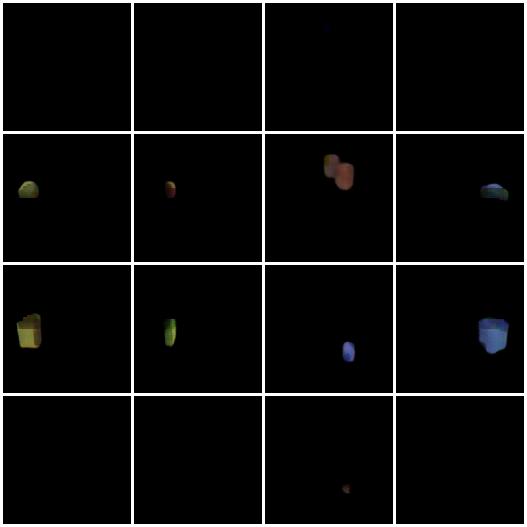}
    \caption{The decomposition of baseline for example in Figure \ref{fig:clevr_single}.}
    \label{fig:clevr_mlpset}
\end{figure*}

\section{Sort-of-CLEVR}
\subsection{Architecture Details}
We reuse the Relational Network and perceptual end implementations from the most popular open source Relational Network implementation \footnote{https://github.com/kimhc6028/relational-networks/tree/74cd9ac0703db01c6268a6515015b98bed3e4602}. The architecture used is 4 convolution layers with 24 channels each, 3x3 kernels, 2 stride, and 1 padding. Between each convolution layer is a relu and a batchnorm layer. 

This results in a (batch x filters x height x width) tensor (batch x 24 x 5 x 5) in our case. Each of the (5x5) cells in the feature map is treated as a size 24 feature map. These represent the ``entities'' which we perform relational reasoning on.

The question is encoded as a binary vector with 11 elements. The first 6 are an one hot encoding of which of the 6 colors the question is about. The next 2 are an one hot encoding of whether the question is relational or non-relational. The last 3 are an one-hot encoding of the 3 question subtypes. 

The question is concatenated to each of the entities, and is passed to the Relational Network. The $g_{\theta}$ (as in \cite{Santoro2017ASN}) is a 4 layer MLP with 256 elements per layer and ReLU non-linearities. A final Linear/ReLU layer is used at the end.

\subsection{Dataset Details}
We use the dataset generator from the same source as the model. \footnote{https://github.com/kimhc6028/relational-networks/tree/74cd9ac0703db01c6268a6515015b98bed3e4602}. However, instead of 10k images we use 100k. With a lower dataset size, both models tend to overfit and make it more difficult to demonstrate the responsibility problem vs general inaccuracy. Our limited experiments still demonstrate that SRN still improves both robustness and performance on the smaller dataset as well.

The configuration used in Robustness (Synthetic) is shown in Figure \ref{fig:robustness_synthetic}. The yellow shape is randomly sampled between a circle/rectangle to avoid degenerate solutions like predicting the opposite of the red shape. All questions are of the form ``What is the color of the furthest shape from the red shape?''.

The dataset used in Robustness (Dataset) is generated with the same code as the one used for training, with 2 exceptions. First, we only keep ``easy'' images - in other words, images where the shape furthest away from the red shape is at least twice as far as the second furthest shape. Second, we enforce all shapes to not be within 15 of the border. Otherwise, the rotations required for checking ``robustness'' would move the shapes outside of the borders.

\subsection{Hyperparameters Considered}
We find that this task tends to be fairly robust to hyperparameters. We tried both 5 steps and 10 steps for inner optimization steps, and did not find a significant difference. The inner learning rate is set to $0.1$ - changing the inner learning rate did not show a significant difference in performance either.

\subsection{Experiment with Cluttered Background}
As an example where traditional approaches might have more difficulty, we try a task with a cluttered background. We fill the background with 5 background rectangles with random color, height, and width. Note that without a task, it would seem like there are now 11 ``objects'' in this scene. For the baseline RN, this task poses significantly more difficulty, even on the non-relational task. After 20 epochs, it reaches 98.5\% accuracy on non-relational questions, and 94.0\% accuracy on relational questions. On the other hand, the SRN still reaches 99.9\% performance on non-relational questions, and 96.0\% accuracy on relational questions.

\begin{figure*}[t!]
    \centering
    \includegraphics[width = 0.3\linewidth]{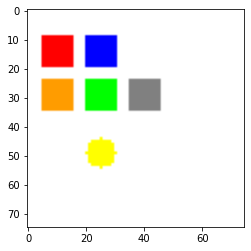}
    \includegraphics[width = 0.3\linewidth]{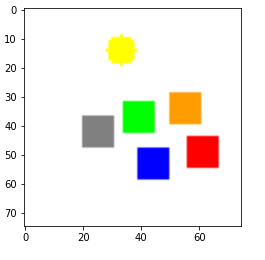}
    \caption{Examples of Robustness (Synthetic). The left one is the original configuration, and the right one is an example of how rotation looks like.}
    \label{fig:robustness_synthetic}
\end{figure*}

\section{World Model Learning}

\subsection{Architecture Details}
As described in the main paper, we used SRN module to refine the original state vectors $S_0$. Data embedding $g$ is generated by flattening the feature maps generated by object extractor, and passing it through a linear layer to 512 dimension (same as the number of hidden units in transition MLP). $\dspnenc$ is the same as in the circles reconstruction experiment. 

\subsection{Training and Hyper-parameters Details}
We use the official C-SWM code \footnote{https://github.com/tkipf/c-swm} and follow most of the training procedures in the paper, except for the evaluation set size, learning rate and number of epochs. We increase the evaluation set size from 100 to 1000 episodes to reduce variance. For learning rate, we test 5e-5, 1e-4, 5e-4, 1e-3, among which 1e-4 has substantially better performance. So we use 1e-4 learning rate for all runs and train until 500 epochs to reach convergence. The comparison of the baseline between original setup and new setup is included in table \ref{tab:baselines}. Note that there is a large variance of performance across epochs. In the main paper table, we report each model's best average performance on each metric, regardless whether it is the same epoch selected. 
In general, the epochs selected are all close to 500. 

For Pong, SRN has inner learning rate $0.045$ and $r = 10$. For Space Invader, SRN has inner learning rate $0.05$ and $r = 10$. We selected these values by starting with a random sampled value and hand-tuning for best eval performance. Specifically, we searched in the range 0.0025 to 1 for inner learning rate and $r = 5, 10$.

\begin{table}[t]
\caption{C-SWM hits at rank 1~\cite{Kipf2020ContrastiveLO} and mean reciprocal rank accuracy with learning rate 1e-4 and 5e-4 on the Atari Pong (Pong) and Space Invaders (Space) datasets (mean and standard deviation over 5 trials). The numbers in the parenthesis after model names are the number of object slots $K$ and epochs.}\label{tab:baselines}
\vspace{-\baselineskip}
\begin{center}
\begin{tabular}{ll ccccccc}
\toprule
&  & \multicolumn{2}{c}{1 Step} & \multicolumn{2}{c}{5 Steps}&\multicolumn{2}{c}{10 Steps} \\ 
\cmidrule(lr){3-4} \cmidrule(lr){5-6} \cmidrule(lr){7-8}
   & Configurations &  H@1  & MRR &  H@1  & MRR &  H@1  & MRR \\
\midrule
\multirow{8}{*}{\rotatebox{90}{ Pong }} 
&1e-4 (5, 500) & $\textbf{63.6}$ {\tiny $\pm 1.6$} & $\textbf{75.8}$ {\tiny $\pm 1.2$} & $\textbf{26.1}$ {\tiny $\pm 1.7$} & $\textbf{43.1}$ {\tiny $\pm 1.7$}  & $\textbf{16.1}$ {\tiny $\pm 1.4$}  & $28.9$ {\tiny $\pm 2.4$}\\
&1e-4 (5, 200) & $45.7$ {\tiny $\pm 4.8$} & $63.5$ {\tiny $\pm 3.7$}& $16.5$ {\tiny $\pm 3.8$}& $32.8$ {\tiny $\pm 4.1$} & $8.6$ {\tiny $\pm 1.9$} & $21.2$ {\tiny $\pm 2.9$} \\ 
&1e-4 (3, 500) & $52.5$ {\tiny $\pm 6.8$} & $68.6$ {\tiny $\pm 4.4$} & $24.0$ {\tiny $\pm 3.5$} & $40.7$ {\tiny $\pm 4.0$} & $11.8$ {\tiny $\pm 3.3$} & $25.1$ {\tiny $\pm 4.7$}\\
&1e-4 (3, 200) & $39.7$ {\tiny $\pm 6.6$}& $59.0$ {\tiny $\pm 5.5$}& $13.9$ {\tiny $\pm 2.6$} & $28.8$ {\tiny $\pm 3.3$}& $6.5$ {\tiny $\pm 1.5$}& $16.5$ {\tiny $\pm 2.0$}\\
&5e-4 (5, 500) & $56.8$ {\tiny $\pm 9.6$}& $71.6$ {\tiny $\pm 6.4$}& $21.2$ {\tiny $\pm 4.9$} & $37.6$ {\tiny $\pm 5.3$}& $13.3$ {\tiny $\pm 3.9$}& $26.5$ {\tiny $\pm 5.0$}\\
&5e-4 (5, 200) & $22.8$ {\tiny $\pm 12.5$}& $42.0$ {\tiny $\pm 13.4$} & $7.1$ {\tiny $\pm 4.6$} & $18.6$ {\tiny $\pm 6.9$}& $4.0$ {\tiny $\pm 2.9$} & $11.8$ {\tiny $\pm 5.8$} \\
&5e-4 (3, 500) & $29.7$ {\tiny $\pm 8.6$}& $49.4$ {\tiny $\pm 8.1$}& $20.4$ {\tiny $\pm 8.2$}& $37.0$ {\tiny $\pm 9.3$}& $14.4$ {\tiny $\pm 5.4$}& $\textbf{30.1}$ {\tiny $\pm 7.2$}\\
&5e-4 (3, 200) & $26.9$ {\tiny $\pm 2.4$}& $46.8$ {\tiny $\pm 2.7$}& $12.4$ {\tiny $\pm 3.5$}& $27.9$ {\tiny $\pm 4.9$}& $10.2$ {\tiny $\pm 4.0$}& $23.6$ {\tiny $\pm 5.9$}\\
\midrule
\multirow{8}{*}{\rotatebox{90}{ Space }}
&1e-4 (5, 500) & $46.9 $ {\tiny $\pm 6.9$} & $65.6 $ {\tiny $\pm 5.2$} & $33.1 $ {\tiny $\pm 12.2$} & $53.6 $ {\tiny $\pm 11.1$} & $21.6 $ {\tiny $\pm 5.2$} & $40.0 $ {\tiny $\pm 5.7$} \\ 
&1e-4 (5, 200) & $37.1$ {\tiny $\pm 8.0$} & $56.8$ {\tiny $\pm 7.3$} & $19.6$ {\tiny $\pm 9.3$} & $38.1$ {\tiny $\pm 12.2$} & $10.0$ {\tiny $\pm 3.5$} & $23.5$ {\tiny $\pm 4.9$}\\
&1e-4 (3, 500) & $\textbf{67.9} $ {\tiny $\pm 3.6$} & $\textbf{80.6} $ {\tiny $\pm 2.7$} & $\textbf{46.7} $ {\tiny $\pm 18.5$} & $\textbf{64.4} $ {\tiny $\pm 17.4$} & $\textbf{30.7} $ {\tiny $\pm 8.6$} & $\textbf{50.6} $ {\tiny $\pm 8.3$} \\
&1e-4 (3, 200) & $54.1$ {\tiny $\pm 7.6$} & $70.6$ {\tiny $\pm 5.7$} & $30.3$ {\tiny $\pm 13.1$} & $50.4$ {\tiny $\pm 14.9$} & $16.8$ {\tiny $\pm 3.1$} & $33.4$ {\tiny $\pm 4.0$}\\
&5e-4 (5, 500) & $56.8$ {\tiny $\pm 3.8$} & $73.4$ {\tiny $\pm 2.9$} & $34.6$ {\tiny $\pm 14.1$} & $53.7$ {\tiny $\pm 13.8$} & $27.9$ {\tiny $\pm 11.4$} & $47.3$ {\tiny $\pm 13.0$}\\
&5e-4 (5, 200) & $40.0$ {\tiny $\pm 9.0$} & $60.8$ {\tiny $\pm 7.4$} & $14.2$ {\tiny $\pm 5.8$} & $31.1$ {\tiny $\pm 7.5$} & $11.5$ {\tiny $\pm 4.0$} & $26.0$ {\tiny $\pm 6.2$}\\
&5e-4 (3, 500) & $56.9$ {\tiny $\pm 3.3$} & $73.3$ {\tiny $\pm 2.0$} & $31.7$ {\tiny $\pm 11.2$} & $51.3$ {\tiny $10.5$} & $23.9$ {\tiny $\pm 14.0$} & $42.3$ {\tiny $\pm 13.8$}\\
&5e-4 (3, 200) & $55.8$ {\tiny $\pm 4.7$} & $72.5$ {\tiny $\pm 3.3$} & $32.8$ {\tiny $\pm 7.8$} & $54.2$ {\tiny $\pm 8.2$} & $15.5$ {\tiny $\pm 11.6$} & $33.5$ {\tiny $\pm 13.9$}\\
\bottomrule
\end{tabular}
\end{center}
\end{table}

\section{Language Reasoning}
We used the official CLUTRR-Baselines \footnote{https://github.com/koustuvsinha/clutrr-baselines} code for all our experiments. For SRN, $\dspnenc$ is the same as in the circles experiment. The inner learning rate is set to $0.001$ and $r = 5$. For GRN, we use scene graph encoder based on \citet{Johnson2018ImageGF} as $\dspnenc$. The inner learning rate is set to $0.0005$ and $r = 10$. The hyper-parameters are selected by starting with a random sampled value and hand-tuning for best eval performance. Specifically, we searched in the range 1e-3 to 1e-5 for l2-penalty, by which the result is affected the most.

\section{Computing Infrastructure and Runtime}
We use PyTorch \cite{Paszke2019PyTorchAI} for all of our experiments. All experiments were run on a SLURM cluster - primarily on a machine with a single RTX 2080TI Nvidia GPU. All runtime results are reported with those.

The effect that SRN has on runtime is dependent on the size of the perceptual end as well as the relational reasoning end. However, the effect that SRN has on runtime is disproportionate to the actual complexity of the models used in SRN, as the inner optimization loop means that SRN must execute the models within it multiple times. 

In the CLUTTR and circle reconstruction experiment, for example, the runtime is dominated by the perceptual end as well as the image reconstruction. As such, SRN increases the time per epoch by ~10\%.

On the other hand, for Sort-of-CLEVR and C-SWM, our experiments used a SRN module that increases runtime by about 2-3x. We note that although this is a significant runtime increase, it is somewhat ameliorated by 2 factors: 1. SRN tends to increases convergence rate. For example, on Sort-of-CLEVR SRN reaches 99\% performance on non-relational questions at epoch 3 vs epoch 10 (final performance for both is 100\%). 2. We have not focused on reducing the runtime of SRN. As we have reused the same FSPool set encoder with the same parameters for most all of our tasks, preliminary experiments suggest that we can improve runtime significantly by reducing the size of this encoder, possibly without affecting performance. Similarly, we can reduce runtime significantly by decreasing the number of inner optimization steps taken.

\section{Datasets and Code Release}

All datasets we used in this paper will be released together with our code, which will also contain the code and reference for generating the datasets.

\end{document}